%% file: stable-initialization.tex
\documentclass{article} 
\usepackage{iclr2027_conference}
\usepackage{times}

\input{math_commands.tex}

\usepackage[activate={true,nocompatibility},final,tracking=true,kerning=true,spacing=true,factor=1100,stretch=10,shrink=10]{microtype}
\microtypecontext{spacing=nonfrench}
\usepackage{amsmath,amssymb,amsfonts, mathtools}
\usepackage{amsthm}
\usepackage{commath}

\usepackage{hyperref}
\usepackage{url}
\usepackage{graphicx}
\usepackage{placeins} 
\usepackage{float}
\usepackage{booktabs}

\DeclareMathOperator{\expect}{\mathbb{E}}

\DeclareMathOperator{\Cos}{\operatorname{Cos}}

\newtheorem{definition}{Definition}
\newtheorem{lemma}{Lemma}
\newtheorem{theorem}{Theorem}

\newtheorem{remark}{Remark}
\newtheorem*{priorwork}{Prior work}
\newtheorem*{ourresult}{Our result}

\title{Stable initialization without the CLT}

\author{Simon Kuang \\
Department of Mechanical and Aerospace Engineering\\
University of California, Davis\\
\texttt{slku@ucdavis.edu} \\
\And
Kyle Chickering \\
\texttt{krchicke@ucdavis.edu} \\
\And
Xinfan Lin \\
Department of Mechanical and Aerospace Engineering\\
University of California, Davis\\
\texttt{lxflin@ucdavis.edu}
}

\usepackage{etoolbox}
\usepackage{comment}
\usepackage{xcolor}
\newtoggle{showbefore}
\newtoggle{showafter}

\togglefalse{showbefore}

\toggletrue{showafter}

\iftoggle{showbefore}{
    
}{
    \excludecomment{before}
}

\iftoggle{showafter}{
    \newenvironment{after}{\color{black}\ignorespaces}{\ignorespacesafterend}
}{
    
}

\usepackage{xcolor}

\iclrfinalcopy 
\begin{document}

\maketitle
\lhead{Preprint}

\begin{abstract}
	Successful training of deep neural networks is highly dependent on the distribution of the initial weights.
	If the weights are too large, network training blows up; if they are too small, the model fails to learn features.
	Stable initialization is the optimal moderation between these two extremes.
	The conventional theory of random networks uses the Central Limit Theorem to control inter-neuron dependencies,
	which introduces distributional approximation error and coupling between layers.
	For networks with sine activations, we derive the uniform-phase initialization, which obviates distributional approximation and fully decouples the layers.
	Ours is the first work to use the sine function's periodic symmetry.
	Models trained with the uniform-phase initialization outperform the state of the art in neural representation tasks like image and audio fitting.
	We find that our untuned models are competitive with the best-tuned baselines from previous work and support $\mu$P width scaling.
\end{abstract}

\section{Introduction}\label{sec:introduction}
Today it is widely accepted that increasing a model's representational capacity by increasing its width $N$ and depth $L$ leads to stronger performance on complex learning tasks \citep{hestness_deep_2017}.
Two decades ago, however, training large models was a serious concern: wide and deep networks are highly sensitive to the initialization of their parameters \cite[\S11.5.1]{hastie_elements_2009} and suffer from exploding or vanishing gradients, which make first-order optimization challenging \citep{glorot_understanding_2010}.
These so-called gradient instabilities motivated the search for \textbf{stable initialization} distributions for the weights.

We study a \textbf{multilayer perceptron} \(f: \mathbb R^{N_\text{in}} \to \mathbb R^{N_\text{out}}\), which is the composition
\(f(x) = (g_\text{out} \circ h \circ g_\text{in})(x)\), comprising
\begin{align*}
	g_\text{in}  & : \mathbb R^{N_\text{in}} \to \mathbb R^{N},  & x & \mapsto \sigma(W_\text{in} x + b_\text{in}),           &  & \text{the {\bfseries input layer};}
	\\
	g_\text{out} & : \mathbb R^{N} \to \mathbb R^{N_\text{out}}, & x & \mapsto W_\text{out} x + b_\text{out},                 &  & \text{the {\bfseries output layer}; and}
	\\
	h            & : \mathbb R^{N} \to \mathbb R^{N},            & x & \mapsto (h_L \circ h_{L-1} \circ \cdots \circ h_1)(x), &  & \text{the {\bfseries hidden layers}, defined by}
	\\
	h_\ell       & : \mathbb R^{N} \to \mathbb R^{N},            & x & \mapsto \sigma(W_\ell x + b_\ell),                     &  & \text{for $\ell \in \{1, 2, \ldots, L\}$}.
\end{align*}

Stable initialization encompasses two questions: (1) how to tailor the hidden layers to the activation function and (2) how to tailor the input and output layers to the task.

\paragraph{Hidden layers}
Early theory required that as \(L, N \to \infty\), the initial network's regularity properties scale as \(O(1)\), i.e.~not ``blow up'' \citep{montavon_efficient_2012,glorot_understanding_2010,he_delving_2015}.
The edge-of-chaos school argues that ``blowdown'' is equally pathological and strengthened this desideratum to \(\Theta(1)\) \citep{poole_exponential_2016,schoenholz_deep_2017,yang_mean_2017,hayou_impact_2019}.
To understand why this requirement is generally nontrivial for hidden layers, we express the Jacobian of \(h\) at a given input:
\begin{align}
	\dpd{}{x} h(x) = \mathrm{diag}(\sigma'(z_L)) W_L \cdots \mathrm{diag}(\sigma'(z_1)) W_1,
	\label{eq:jacobian-chain-rule}
\end{align}
where $z_\ell$ is the preactivation at hidden layer $\ell$.
The activation slopes \(\mathrm{diag}(\sigma'(z_\ell))\), \(\ell \in \{1,\ldots,L\}\), depend on the preactivation distribution, which in turn depends on the previous layer's input and weights.

Managing these dependencies demands sophisticated analysis, usually in four steps:
\begin{enumerate}
	\item Identify the criteria to stabilize, usually preactivation scales and gradient norms \citep{montavon_efficient_2012,glorot_understanding_2010,saxe_exact_2014,he_delving_2015,hanin_how_2018,pennington_resurrecting_2017,pennington_spectrum_2018,xiao_dynamical_2018,hayou_stable_2021,yang_tensor_2022}.
	\item Calculate the criteria, approximately or exactly, for a location-scale parametric family, usually Normal, of preactivation distributions \citep{poole_exponential_2016,schoenholz_deep_2017,klambauer_self-normalizing_2017,hayou_impact_2019}.
	      This step usually involves calculating or approximating \(\expect \sigma(Z)\),
	      \(\Var \sigma(Z)\), \(\expect \sigma'(Z)\), and \(\Var \sigma'(Z)\), where \(Z\) is a random variable that approximates the preactivation.
	\item Choose distributions for \(W_\ell\) and \(b_\ell\) such that the preactivation distributions satisfy a layerwise fixed-point recursion, provided that they are Normally distributed \citep{montavon_efficient_2012,glorot_understanding_2010,he_delving_2015,mishkin_all_2016,hayou_stable_2021}.
	\item Apply the Central Limit Theorem to argue that in the large-\(N\) limit, the preactivations are approximately Normal \citep{poole_exponential_2016,schoenholz_deep_2017,yang_mean_2017,lee_deep_2017,jacot_neural_2018,yang_tensor_2019}.
	      This closes the layerwise recursion.
\end{enumerate}
The Central Limit Theorem is the weakest link in this theory.
It incurs an error of \(O(N^{-1})\), and narrow networks are appreciably non-Gaussian at initialization \citep{wolinski_gaussian_2022}.
One solution, pursued in the monograph by \citet{roberts_principles_2022}, is to expand to second order, reducing the asymptotic error to \(O(N^{-2})\).

We focus on the activation function \(\sin(x)\).
The conventional route is taken by \citet{combette_new_2026}, who instantiate the fixed-point recursion (Step 3) as a transcendental equation involving the Lambert W function.
The second-order theory of \citet[\S5.3.3,\S9.3]{roberts_principles_2022} classifies \(\sin\) in the same universality class as \(\tanh\).
Altogether, previous work on stable-initialization theory for \(\sin(x)\) uses general properties of \(\sin(x)\), such as its moments, Taylor series, and range, but---crucially---has ignored its periodicity.

\paragraph{Input and output layers}
The first and last layers of a neural network deserve their own attention.
Whereas stable initialization theory is interested in the asymptotic \emph{slope} (namely, zero, for edge-of-chaos initialization) of the network's behavior with respect to \(L\) and \(N\),
the absolute scale of the input and output weights and biases determines the \emph{intercept}:
no matter how large \(L\) and \(N\) are, \(f(10x)\) is a very different function than \(f(x)\).
Sinusoidal neural networks are known to exhibit a fragile dependence on the weight scale in general \citep{yeom_fast_2025} and input weights in particular \citep{sitzmann_implicit_2020,combette_new_2026}.
Many works tune a hyperparameter \(\omega_0\) that scales the inputs by an isotropic base frequency.
(The Nyquist frequency can be motivated by sensor resolution \citep{yuce_structured_2022}, but it follows the sensor rather than the underlying field.)
\citet{alsakabi_ja-siren_2026} use a root-finding algorithm to match the initial network's spectrum to the discrete sine transform of the field, which is optimal in a sense, but requires rich information about the training data.
\citet{tancik_learned_2021} target the input frequency of a Fourier feature network as a meta-learning task.

\subsection{Summary of our method}
Conventional stable initialization theory models preactivations as distributed on \(\mathbb R^N\) and centered at 0,
which necessitates small or zero biases.
However, because sine is a periodic function, the preactivation may equivalently be viewed as wrapped on the torus \([0, 2\pi]^N\), which has a unique translation-invariant probability measure.
By sampling large biases, we spread the preactivations uniformly around the torus \emph{irrespective of the location and scale of the previous layer}, which corresponds to 
drawing biases from \(\mathcal N(0, \infty)\) in a conventional stable initialization (Appendix~\ref{app:cvp26-comparison}). 
\begin{definition}
	The \emph{uniform-phase} initialization for hidden layers
	\(x \mapsto \sin(Wx + b)\)
	is
	\begin{align*}
		W_{ij} & \sim \mathcal{N}(0, 2/N)    &  & \text{independently, and} \\
		b_i    & \sim \mathcal{U}([0, 2\pi]) &  & \text{independently.}
	\end{align*}
\end{definition}

As a consequence, miraculously, all of the entries of all of the matrices in \eqref{eq:jacobian-chain-rule} become not only jointly independent but also independent of the input (Lemma~\ref{lem:independent-factorization}).
The essence is summarized by the following fact:
\begin{quote}
	Let \(X\) have \emph{any} distribution on \(\mathbb R\), and let \(b \sim \mathcal U([0, 2\pi])\) be independent of \(X\).

	Then \(\cos(X + b)\) is independent of \(X\).
\end{quote}

We state and prove the neural network version (Lemma~\ref{lem:independent-factorization})
and then deploy it to prove our main result (Theorem~\ref{thm:general}) on the stability of hidden layers.
Our hidden-layer theory facilitates
an auxiliary result on initializing the input and output layers (Section~\ref{sec:input-output-layers}),
which lends interpretability to the input and output scales by relating them to non-asymptotic average-case regularity properties of the initial network.

Altogether, for every fixed input \(x\), our initialization distribution guarantees:
\begin{equation*}
	\expect J^\intercal J = \expect J J^\intercal = I,
\end{equation*}
simultaneously for all Jacobians \(J\) between \emph{any} two hidden layers, and
\begin{equation*}
	\expect f(x), \quad \Cov f(x), \quad \text{and} \quad \expect \nabla f(x) \del{\Cov f(x)}^{-1} (\nabla f(x))^\intercal
\end{equation*}
all attain prescribed values exactly.
They can be imposed as an inductive bias or moment-matched to data.

\subsection{Contributions}
We propose a simple and---to our knowledge, without precedent---exact solution to stable initialization at any width and depth that harnesses the periodicity of sinusoidal activation.
We also propose an exact moment-matching criterion for zero-shot tuning of the input and output layer initialization.
Our method is architecture-agnostic and adapts to the spatial resolution and output statistics of the neural representation task's target.
Untuned, it outperforms fine-tuned baselines for sinusoidal network initialization.

We validate the conventional wisdom on the effectiveness of stable initialization by comparing our initialization to stable initializations in the literature.
On most examples, our method meets or exceeds the bar set by previous work.
We also test the learning-rate scaling of the maximal-update parametrization (\(\mu\)P) for neural representations and question the necessity of its architecture-dependent last-layer scaling.

\subsection{Related work}
\paragraph{Sinusoidal activations and neural representation}
The sine function satisfies the hypotheses of universal approximation theory in neural networks; in fact, similar guarantees can be obtained using the older and more robust theory of Fourier analysis \citep{parascandolo_taming_2017}.
However, multilayer networks with sinusoidal activation have proved remarkably effective for the task of \textbf{neural representation}: learning real-world signals represented as functions of an input coordinate, e.g.~audio as a map from time to sound pressure, or an image as a map from spatial coordinates to pixel values \citep{sitzmann_implicit_2020,dupont_coin_2021,avidan_implicit_2022}.
A related idea (Fourier features) is to use sinusoids in the input layer, but with fixed isotropic weights \citep{tancik_fourier_2020,mildenhall_nerf_2020}.

\paragraph{Initializing sinusoidal networks}
\citet{sitzmann_implicit_2020} (henceforth SM20) attempt to impose zero mean and unit variance on the preactivation distribution but do not account for gradient blowup with network depth.
Although input-frequency tuning partly mitigates this weakness from an NTK perspective \citep{belbute-peres_simple_2023},
\citet{combette_new_2026} (henceforth CVP26) are the first to carry out the complete stable-initialization procedure to control both preactivations and gradients.
\citet{novello_tuning_2025,yeom_fast_2025} are explicitly interested in controlling the frequency power spectrum of the initialized network.
\citet{alsakabi_ja-siren_2026} deterministically match it to the training data.
\citet{ben-shabat_digs_2022} propose an initialization geared specifically toward signed distance functions.

\paragraph{Scaling, feature learning, and \(\mu\)P}
Maximal update parametrization (\(\mu\)P) is a theory for scaling hyperparameters with width \(N\) at a fixed depth (the main use case is to tune a model's hyperparameters at a small scale and subsequently transfer them to a much wider model) \citep{yang_tensor_2022,yang_spectral_2024,dey_dont_2026}.
But it only gives the slope and not the intercept; for sine-activated networks, the absolute scales depend on the task and have an outsized effect on training \citep{yeom_fast_2025}.
While the predictions of \(\mu\)P are asymptotic \(\Theta(\cdot)\)- and \(O(\cdot)\)-type results, our results are exact equalities for any width and depth.
\(\mu\)P's learning-rate scaling is orthogonal to our work, but we test it nevertheless on our initializations.
\(\mu\)P's initialization scaling agrees with our exact parameters in all but the final layer; we discuss this subtlety in Section~\ref{sec:mup}.
\section{Theory of the uniform-phase initialization for hidden layers}\label{sec:theory}
Specializing the hidden map \(h\) from Section~\ref{sec:introduction} to \(\sigma = \sin\), let \(x_0 = x\) and write the layerwise pre- and post-activations as
\begin{align}
	\label{eq:network-definition}
	z_\ell = W_\ell x_{\ell-1} + b_\ell, \qquad x_\ell = \sin(z_\ell), \qquad \ell \in [1\ldots L],
\end{align}
so that \(h(x) = x_L\).
Introduce the notation
\(\Cos x = \operatorname{diag}(\cos(x))\), where \(\cos\) is applied elementwise.
All expectations in this section are over the initialization distribution, with \(x_0\) arbitrary.

We now state and prove a few useful properties of the uniform-phase initialization.

\begin{lemma}[Independent Jacobian factorization]
	\label{lem:independent-factorization}
	In the uniform-phase initialization, the terms
	\(\{W_\ell\}_{\ell=1}^L \cup \{\cos z_\ell\}_{\ell=1}^L\) are independent.
\end{lemma}
\begin{proof}
	Let \(\phi_1, \ldots, \phi_{L}: \mathbb R^{N} \to \mathbb R\) and \(\psi_1, \ldots, \psi_L: \mathbb R^{N\times N} \to \mathbb R\) be arbitrary indicator functions of measurable sets.
	\begin{align}
		\MoveEqLeft
		\expect \prod_{\ell=1}^{L} \phi_\ell(\cos(z_\ell)) \psi_\ell(W_\ell)
		\notag
		\\
		 & = \expect \phi_L(\cos(W_L x_{L-1} + b_L)) \psi_L(W_L) \prod_{\ell=1}^{L-1} \phi_\ell(\cos(W_\ell x_{\ell-1} + b_\ell)) \psi_\ell(W_\ell)
		\notag
		\\
		 & = \expect
		\underbrace{
			\expect_{b_L} \sbr{\phi_L(\cos(W_L x_{L-1} + b_L))}
		}_{= \text{constant}}
		\psi_L(W_L)
		\prod_{\ell=1}^{L-1} \phi_\ell(\cos(W_\ell x_{\ell-1} + b_\ell)) \psi_\ell(W_\ell)
		\tag{conditioning on all randomness other than \(b_L\)}
		\\
		 & =
		\expect \sbr{\phi_L(\cos(W_L x_{L-1} + b_L))}
		\expect
		\psi_L(W_L)
		\prod_{\ell=1}^{L-1} \phi_\ell(\cos(W_\ell x_{\ell-1} + b_\ell)) \psi_\ell(W_\ell)
		\tag{\(\star\)}
		\\
		 & =
		\expect \sbr{\phi_L(\cos(W_L x_{L-1} + b_L))}
		\expect
		\psi_L(W_L)
		\expect
		\prod_{\ell=1}^{L-1} \phi_\ell(\cos(W_\ell x_{\ell-1} + b_\ell)) \psi_\ell(W_\ell)
		\tag{\(W_L\) is independent}
		\\
		 & =
		\prod_{\ell=1}^{L} \expect  \phi_\ell(\cos(z_\ell)) \expect \psi_\ell(W_\ell)
		\tag{induction on \(L\)}
	\end{align}
	The crucial step, (\(\star\)), uses the fact that \(W_L x_{L-1} + b_L\) is uniformly distributed modulo \(2\pi\) for any fixed \(W_L x_{L-1}\).
	Averaging over \(b_L\) therefore yields a constant, which can be pulled out of the expectation.
\end{proof}

\begin{lemma}[Moments of uniform-phase sinusoids]
	\label{lem:sinusoid-moments}
	Let \(b\) be uniformly distributed on the torus \([0, 2\pi]^N\). Then
	\(\expect \cos(b) = 0\) and \(\expect \sbr{\cos(b)\cos(b)^\intercal} = \frac12 I\).
\end{lemma}
\begin{proof}
	By circular symmetry in each coordinate, \(\expect \cos(b) = 0\).
	Because the coordinates are independent, the off-diagonal entries of \(\expect \sbr{\cos(b)\cos(b)^\intercal}\) are zero.
	On the diagonal, \(\expect \cos^2 b_1 = \expect \sin^2 b_1\) by circular symmetry, while \(\expect \cos^2 b_1 + \expect \sin^2 b_1 = 1\) by the Pythagorean identity; therefore \(\expect \cos^2 b_1 = \frac12\).
\end{proof}

Combining these inductively proves our main result on hidden layers.

\begin{theorem}\label{thm:general}
	Let \(0 \leq s \leq t \leq L\) be integers.
	Then for all inputs \(x_0\), the hidden layers \eqref{eq:network-definition} satisfy
	\begin{align}
		\expect \del{\dpd{x_t}{x_s}}
		\del{\dpd{x_t}{x_s}}^\intercal
		 & = I,
		 &
		\expect \del{\dpd{x_t}{x_s}}^\intercal
		\del{\dpd{x_t}{x_s}}
		 & = I.
	\end{align}
\end{theorem}
Write \(\delta_s^t \coloneqq \dpd{x_t}{x_s}\).
By the chain rule, \(\delta_s^t\) satisfies both a forward recursion in \(t\) (with \(s\) fixed) and a backward recursion in \(s\) (with \(t\) fixed):
\begin{subequations}
	\begin{align}
		\delta_s^t & = \Cos z_t W_t \delta_s^{t - 1},       & \delta_s^s & = I,
		\label{eq:general-forward}
		\\
		\delta_s^t & = \delta_{s+1}^t \Cos z_{s+1} W_{s+1}, & \delta_t^t & = I.
		\label{eq:general-backward}
	\end{align}
\end{subequations}
Note that \(\delta_s^t\) depends only on \(x_s\), \(\{W_\ell\}_{\ell=s+1}^t\), and \(\{b_\ell\}_{\ell=s+1}^t\); we shall use this fact to argue for independence in later steps.
\begin{proof}[Proof that \(\expect \delta_s^t (\delta_s^t)^\intercal = I\)]
	We prove this by induction on \(t\), with \(s\) fixed.
	The base case \(t = s\) is immediate from \eqref{eq:general-forward}.
	For the inductive step, we assume that \(\expect \delta_s^{t-1} (\delta_s^{t-1})^\intercal = I\) and want to show that \(\expect \delta_s^t (\delta_s^t)^\intercal = I\).
	By \eqref{eq:general-forward},
	\begin{align}
		\delta_s^t (\delta_s^t)^\intercal
		 & = \Cos z_t W_t \delta_s^{t-1} (\delta_s^{t-1})^\intercal W_t^\intercal \Cos z_t.
		\intertext{According to Lemma~\ref{lem:independent-factorization}, \(\delta_s^{t-1}\), \(W_t\), and \(\Cos z_t\) are mutually independent.
			Therefore, we may evaluate this expectation from the inside out by repeatedly conditioning on the outer terms:}
		\expect \delta_s^t (\delta_s^t)^\intercal
		 & = \expect \Cos z_t W_t \del{
			\expect \delta_s^{t-1} (\delta_s^{t-1})^\intercal
		} W_t^\intercal \Cos z_t
		\\
		 & = \expect \Cos z_t W_t W_t^\intercal \Cos z_t
		\\
		 & = \expect \Cos z_t \del{\expect W_t W_t^\intercal }\Cos z_t
		\\
		 & = 2 \expect \Cos z_t\Cos z_t
		\\
		 & = I
	\end{align}
	by Lemma~\ref{lem:sinusoid-moments}.
\end{proof}
\begin{proof}[Proof that \(\expect (\delta_s^t)^\intercal \delta_s^t = I\)]
	We prove this by downward induction on \(s\), with \(t\) fixed.
	The base case \(s = t\) is immediate from \eqref{eq:general-backward}.
	For the inductive step, we assume that
	\(\expect (\delta_{s+1}^t)^\intercal \delta_{s+1}^t = I\) and want to show that \(\expect (\delta_s^t)^\intercal \delta_s^t = I\).
	By \eqref{eq:general-backward},
	\begin{align*}
		(\delta_s^t)^\intercal \delta_s^t
		 & = W_{s+1}^\intercal \Cos z_{s+1} (\delta_{s+1}^t)^\intercal \delta_{s+1}^t \Cos z_{s+1} W_{s+1}.
		\intertext{Using the same conditioning tactic as above, }
		\expect (\delta_s^t)^\intercal \delta_s^t
		 & = \expect W_{s+1}^\intercal \Cos z_{s+1} \del{
			\expect (\delta_{s+1}^t)^\intercal \delta_{s+1}^t
		} \Cos z_{s+1} W_{s+1}
		\\
		 & = \expect W_{s+1}^\intercal \Cos z_{s+1} \Cos z_{s+1} W_{s+1}
		\\
		 & = \expect W_{s+1}^\intercal \del{\expect \Cos z_{s+1} \Cos z_{s+1}} W_{s+1}
		\\
		 & = \frac12 \expect W_{s+1}^\intercal W_{s+1}
		= I.
	\end{align*}
\end{proof}

This theorem guarantees that the Jacobian (or gradient) between any two layers is isotropic, and its average squared singular value is 1.
In the backpropagation direction, early training loss gradients neither grow nor decay on average as they propagate backward through the network.
In the forward direction, the network, as a mathematical function that maps inputs to outputs, possesses zeroth-order (range) and first-order (slope) properties that are independent of width and depth.

\section{Moment matching for input and output layers}\label{sec:input-output-layers}
Thus far, we have dealt only with \(h\), the network's hidden-layer map, which has \(N\) inputs, \(N\) outputs, and depth \(L\).
To complete the network \(f = g_\text{out} \circ h \circ g_\text{in}\) from \(\mathbb R^{N_\text{in}}\) to \(\mathbb R^{N_\text{out}}\) as defined in the introduction, we must initialize the input layer \(g_\text{in}\) and the output layer \(g_\text{out}\).
We draw \(b_\text{in}\) from a uniform distribution on \([0, 2\pi]\), as we do for the hidden-layer biases.
This establishes translation invariance and eliminates the need to center the data.

For target moments \(\mu\), \(\Sigma \succ 0\), and \(\Omega \succeq 0\), we initialize the remaining three parameters \(W_\text{in} \in \mathbb R^{N\times N_\text{in}}\), \(W_\text{out} \in \mathbb R^{N_\text{out}\times N}\), and \(b_\text{out} \in \mathbb R^{N_\text{out}}\) by moment matching:
\begin{subequations}
	\begin{align}
		\mu    & = \expect f(x) = \expect b_\text{out},                                                                                         &  & \text{the output mean,}
		\\
		\Sigma & = \Cov f(x) = \tfrac{1}{2}\expect W_\text{out} W_\text{out}^\intercal,                                                         &  & \text{the output covariance, and}
		\\
		\Omega & = \expect \dpd{f(x)}{x}^\intercal \Sigma^{-1} \dpd{f(x)}{x} = \frac{N_\text{out}}{N}\expect W_\text{in}^\intercal W_\text{in}, &  & \text{the normalized structure tensor.}
	\end{align}
\end{subequations}
Because the uniform-phase initialization is translation invariant, these moments may be interpreted as being evaluated at any particular \(x\) or over a representative distribution of \(x\).
After \(\mu\), \(\Sigma\), and \(\Omega\) are either prescribed or estimated from data, we match these target moments by setting
\begin{subequations}
	\begin{align}
		b_\text{out}   & = \hat \mu,                                                                  \\
		W_{\text{out}} & \sim \mathcal N(0, 2\hat \Sigma / N) \text{ columns},        &  & \text{and} \\
		W_{\text{in}}  & \sim \mathcal N(0, \hat \Omega / N_\text{out}) \text{ rows}.
	\end{align}
\end{subequations}
The key idea is to use \(\Omega\) to calibrate the input scale in the ``physical'' (that is, nonfrequency) domain.
By initializing \(W_\text{out}\) and \(W_\text{in}\) in this way, we nondimensionalize both the inputs and the outputs for the hidden section of the neural network.

In prior work \citep{sitzmann_implicit_2020,combette_new_2026}, the input weights are scaled using a ``characteristic frequency'' \(\omega_0\): the Nyquist frequency (in audio), \(\sim 30\) periods (in images), \(\sim 2\) periods (in physics-informed neural networks).
However, this and other scales have an outsized effect on learning and require tuning \citep{belbute-peres_simple_2023,yeom_fast_2025}.
By contrast, matching the observed value of \(\Omega\) allows the network to adapt to the regularity characteristics of the data.

\section{Numerical experiments}
Stable initialization is a mathematical definition accompanied by a scientific hypothesis:
a stably initialized network will perform better than an unstably initialized network after thousands of epochs under the same training operator.
Our construction, which exactly calibrates initialization-ensemble second moments at any depth and width, establishes the ideal laboratory conditions in which to test this hypothesis.
Appendix~\ref{app:experimental-details} gives full details of the tasks, preprocessing, model implementations, optimization, metrics, and hyperparameter sweeps.

\subsection{Image fitting}
\label{sec:image-fitting}
We fit the Cameraman image on a $128 \times 128$ training grid and evaluate it on a $512 \times 512$ grid to assess generalization.
We compare the following architectures (details in Appendix~\ref{app:experimental-details}):
\begin{description}
	\item[SIREN (UniformPhase-MM)] Our proposed initialization, which uses uniform-phase biases and moment-matched input and output layers, estimates $\mu$, $\Sigma$, and $\Omega$ from the training data, and requires no hand-tuned $\omega_0$.
	\item[SIREN (CVP26 $\sigma_a\!=\!0$)] The edge-of-chaos initialization of \citet{combette_new_2026}.
	\item[SIREN (SM20)] The original initialization of \citet{sitzmann_implicit_2020}.
	\item[Tanh (PE-Xavier)] A Tanh MLP with random Fourier-feature encoding \citep{tancik_fourier_2020} and Xavier weights.
	\item[ReLU (Kaiming), SiLU (Xavier), GeLU (Xavier)] Standard MLPs with the corresponding activation and initialization.
\end{description}

\begin{figure}[t]
	\centering
	\includegraphics[width=\textwidth]{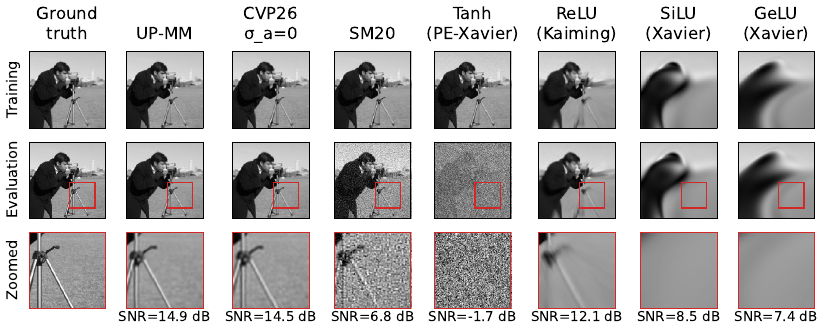}
	\caption{Image fitting on the Cameraman image. Columns: the ground truth followed by one column per initialization, with the evaluation SNR beneath each. Top row: $128 \times 128$ training image. Middle row: $512 \times 512$ test image. Bottom row: zoomed region (red square in the middle row).}
	\label{fig:cameraman}
\end{figure}

The uniform-phase, moment-matched initialization achieves the lowest evaluation MSE by a small margin (Table~\ref{tab:cameraman} in Appendix~\ref{app:image-fitting-results}), followed by CVP26 (\(\sigma_a = 0\)) and the remaining replications of \citet{combette_new_2026}.
Visually, the top two methods both preserve fine details without introducing excessive noise
(\figref{fig:cameraman}).

\subsection{Audio fitting}
We use the same architectures in a replication of the Bach audio experiment of \citet[Appendix 8]{sitzmann_implicit_2020}
\footnote{
	The signal is a $7$\,s mono audio signal at $44.1$\,kHz.
	The time coordinate is scaled to $[-100, 100]$, the audio is normalized by $\max|y|$, and the network is trained and evaluated on the \emph{full} signal (no held-out test split).
	All networks have $L = 5$ hidden layers of width $N = 256$ and are trained for $5000$ epochs with Adam at a learning rate of $5 \times 10^{-5}$.
	The Tanh network uses random Fourier features with $\sigma = 10$.
}.
For the CVP26 and SM20 SIRENs, we replicate the input frequency $\omega_0 = 30$ used by \citet{sitzmann_implicit_2020}.
For our SIREN, we use the moment-matched initialization, which estimates $\hat\Omega$ from the data and requires no hand-tuned frequency.
The SM20 SIREN attains the highest fidelity, with an SNR of $31.15$~dB, but
our moment-matched SIREN reaches $29.33$~dB \emph{without} any hand-tuned frequency parameter; its input scale is automatically calibrated through $\hat\Omega$.
Our SIREN's SNR exceeds CVP26's by nearly 20~dB, and the non-sinusoidal networks (ReLU, SiLU, GeLU) collapse to near-zero output (\figref{fig:audio} in Appendix~\ref{app:image-fitting-results}).

\subsection{Narrow networks}
We repeat the experiment of \S\ref{sec:image-fitting} but reduce the width to \(N = 16\) to stress the infinite-width assumption underlying the CVP26 and SM20 initializations.
The differences between initializations are amplified as a result.
Our moment-matched SIREN remains the best performer across all metrics (SNR: 10.28~dB), and the gap to CVP26 widens to 6.8~dB (\figref{fig:cameraman-narrow} and Table~\ref{tab:cameraman-narrow} in Appendix~\ref{app:image-fitting-results}).

\subsection{Moment matching (zero-shot) versus \texorpdfstring{\(\omega_0\) tuning}{omega0 tuning} (hyperparameter search)}
The hyperparameterizations of CVP26 and SM20 call for an isotropic ``characteristic frequency'' \(\omega_0\) to scale the inputs.
Prior work suggests \(\omega_0 = 64\pi/\sqrt{6}\) for a domain of \([-1, 1]\).
By contrast, our method achieves zero-shot calibration to the anisotropic statistics of the image field; it is invariant to domain translation and equivariant under domain scaling, sampling resolution, and output scaling.
We compare our method's zero-shot tuning to a detailed grid search for prior methods' \(\omega_0\).

Using the same architecture and training as \S\ref{sec:image-fitting}, we evaluate on four \textbf{grayscale} images from \texttt{scikit-image} (Brick, Camera, Gravel, and Clock) and four color images (Hubble Deep Field, Skin, Astronaut, and Coffee), in both \textbf{RGB} and \textbf{CIELAB} color spaces.
For each baseline SIREN initialization (SM20, CVP26 \(\sigma_a\!=\!0\), CVP26 \(\sigma_a\!=\!1\)), we sweep the input-layer frequency \(\omega_0\) over four powers of 10 on a logarithmic grid of 21 values centered at prior work's \(\omega_0\).
The estimated moment-matched parameters and per-image best test errors are reported in Appendix~\ref{app:omega-sweep}, and the best-\(\omega_0\) images in the Supplementary Material, Appendix~\ref{app:best-omega-figures}.

\begin{figure}[t]
	\centering
	\includegraphics[width=\textwidth]{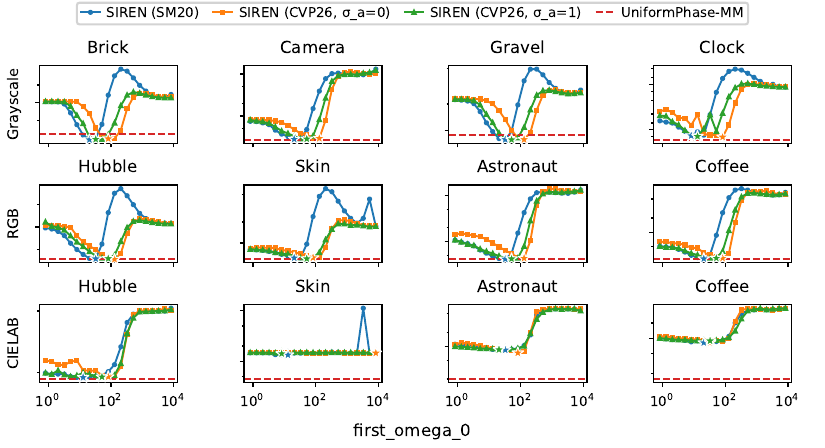}
	\caption{Grayscale and color \(\omega_0\) sweep. Test MSE versus input-layer frequency \(\omega_0\); rows are the grayscale, RGB, and CIELAB experiments, and columns are the four images in each. A star marks the minimum of every sweep, and the dashed red line marks the uniform-phase MM test MSE. Each panel has its own logarithmic vertical scale; for absolute errors, see Tables~\ref{tab:omega-sweep-gray}, \ref{tab:omega-sweep-color}, and \ref{tab:omega-sweep-lab}.}
	\label{fig:omega-sweep-all}
\end{figure}

Across the \textbf{grayscale} and \textbf{RGB} experiments, the moment-matched initialization achieves the lowest test MSE on 5 of 8 images (Camera, Clock, Skin, Astronaut, and Coffee).
On the remaining 3 images (Brick, Gravel, and Hubble), it is competitive, losing to the best-tuned baseline by 19\%, 12\%, and 1.5\%, respectively.
(On Brick and Gravel, we attribute the high error to an overemphasis on high frequencies; see \figref{fig:best-grayscale}.)
Even in these cases, the moment-matched initialization is no worse than mistuning \(\omega_0\) by \(20\%\).

The baselines are all capable of achieving similar performance, but at \(\omega_0\) values that vary widely across images.
For any given method and image, the loss landscape of \(\omega_0\) is nonconvex and contains local minima.
The general trend is that SM20 performs best at the smallest \(\omega_0\), followed by CVP26 (\(\sigma_a=1)\), followed by CVP26 (\(\sigma_a=0\)).
This ordering confirms an established finding \citep{combette_new_2026} that SM20 is biased toward high frequencies and large gradients, followed by CVP26 (\(\sigma_a=1\)) and CVP26 (\(\sigma_a=0\)), in that order.

The situation differs when we fit in three-dimensional \textbf{CIELAB} color space, which is related to RGB space by a nonlinear transformation that approximates human perception.
Perceived colors in CIELAB space are represented by \(L^* \in [0, 100]\) and \(a^*, b^* \in [-128, 127]\).
This transformation preserves the spatial structure of the images but drastically alters the mean and (anisotropic) covariance of the color channels.

For all CIELAB images, our initialization achieves the best fit; the improvement is slight on Hubble and ranges from a factor of 3.0 to 5.2 on the other images.
We attribute this to the fact that it correctly adapts to the distribution of the targets as well as to the typical resolution of spatial detail: the \(\hat\Omega\) structure tensors of CIELAB images are very close to those of RGB images (Appendix~\ref{app:omega-sweep}).

\subsection{Width scaling using \texorpdfstring{\(\mu\)P}{muP}: initialization and learning rates}
\label{sec:mup}
In this section, we test how the optimal learning rate depends on width \(N\).
The universality theory of \(\mu\)P is based on three desiderata \citep[Appendix J.2]{yang_tensor_2022}:
\begin{enumerate}
	\item Every (pre)activation vector in a network should have $\Theta(1)$-sized coordinates. \label{item:mup-activations}
	\item Neural network output should be $O(1)$. \label{item:mup-output}
	\item All parameters should be updated as much as possible (in terms of scaling in width) without leading to divergence. \label{item:mup-updates}
\end{enumerate}
Desideratum~\ref{item:mup-activations} is satisfied by the uniform-phase component of our initialization.
Desideratum~\ref{item:mup-output} is satisfied as \(\Theta(1)\) by the moment-matching property of our initialization; this differs from \(\mu\)P's own last-layer prescription (Remark~\ref{rem:mup-last-layer}).
Desideratum~\ref{item:mup-updates} concerns the learning rate, which is orthogonal to the focus of this paper;
for the Adam optimizer, which we use, \citet{yang_tensor_2022} prescribe scaling the learning rate as \(1/N\) for non-output layers.
We test whether this prescription transfers the optimal learning rate across width for sinusoidal networks, and whether it interacts with our moment-matched initialization.
(To our knowledge, no prior work has validated \(\mu\)P for either sinusoidal MLPs or the neural representation task.)

We repeat the experiment of \secref{sec:image-fitting} for each of the four initializations at seven widths \(N \in \{8, 16, 32, 64, 128, 256, 512\}\), using each initialization's best \(\omega_0\) from a width-\(64\) sweep.
At each width, we sweep the Adam learning rate over powers of two, \(\eta \in \{2^{-18}, \dots, 2^{2}\}\).
\begin{description}
	\item[standard] Plain Adam at the swept \(\eta\) with the initialization left unmodified.
	\item[\(\mu\)P (LR)] The \(\mu\)P learning-rate scaling above; the initialization is unmodified.
	\item[\(\mu\)P (LR+LL)] The same learning-rate scaling, with the last-layer weights also multiplied at initialization by \(\sqrt{64/N}\).
\end{description}

Under standard scaling, the loss curve shifts left with width, and under both \(\mu\)P-informed scalings, the loss curve as a function of learning rate deepens with increasing width.
These observations validate \(\mu\)P's predictions.
But the dominant pattern is the presence of a phase transition immediately to the right of the optimal learning rate, where the loss dramatically increases to a null value (Figure~\ref{fig:mup-lr-train}).
In both \(\mu\)P scalings, the phase transition and the optimal learning rate both shift right with width.
While we have not observed optimal hyperparameter transfer, both \(\mu\)Ps result in a loss that decreases with increasing width at every learning rate.

\(\mu\)P's last-layer rescaling in \textbf{LR+LL} does not appreciably differ from the \textbf{LR}-only scaling.
This supports our theory that, as far as initialization is concerned, moment matching, which is architecture-agnostic, is as effective for learning-rate transfer as \(\mu\)P, which is architecture-aware.

\section{Conclusion}
Stable initialization of sinusoidal networks rests on two pillars: hidden layer stability and input-output layer tuning.
Whereas the conventional theory for hidden layers requires subtle recursions and infinite-width analysis in order to control dependencies between neurons and weights,
our method renders all preactivations independent and identically distributed, allowing network gradient moments to be obtained exactly by a short computation.
We simplify the theory and give the strongest possible guarantee, one that holds for all widths and depths.

Whereas the conventional theory for input-output tuning relies on hand-tuning in the frequency domain, we take a physical-domain perspective on nondimensionalization: canceling the input and output scales using data-driven estimates of the neural field's covariance and structure tensor.
The result is zero-shot initialization tuning on par with a grid search over prior hyperparameterizations.

Our work suggests that sine may be underappreciated as an activation function in machine learning, and that the uniform-phase property may prove fruitful in diverse applications beyond neural representation.

\clearpage
\subsection*{AI use statement}
(Required disclosure) \textbf{In this work, we used AI tools} to
implement methods (replicate CVP26 and SM20, originally written in PyTorch, in JAX; implement our proposed initialization and evaluation harness).
\textbf{We did not use AI tools} to
generate synthetic data sets, help develop theoretical models or conceptual frameworks, formulate mathematical claims, provide critical ingredients for proving mathematical claims, assist in the writing of proofs, propose or refine hypotheses, design or provide feedback on research  methodology or experiments, assist with translation, clean and reformat dataset, support qualitative and thematic data analysis,
or interpret results.

(Recommended disclosure)
\textbf{Additionally, we used AI tools} to summarize or analyse existing literature, discover research topics or identify gaps, search for information, copyedit for grammaticality, and identify relevant literature.

\textbf{We have reviewed all AI-assisted work.}
We carefully proofread AI-generated code for correctness to rule out bugs and reward-hacking tendencies.
Furthermore, we have verified that our replications are within numerical tolerance of prior work.
In the main body of the paper, our use of AI in writing assistance is strictly limited to technical copyediting; wording and stylistic choices are the product of human authorship.
The problem description and reproducibility appendices are drafted by AI and checked for accuracy.
We prepared the bibliography manually using a reference manager and have verified the relevance of each citation.
\textbf{We take responsibility for the final content of this work,
	including text, claims or artifacts produced with the aid of generative AI.}

\subsection*{Ethics statement}
We consider that this work does not raise questions respecting the Code of Ethics.

\subsection*{Reproducibility statement}
Appendix~\ref{app:experimental-details} specifies the data, preprocessing, estimators, initialization distributions, architectures, optimization, random seeds, metrics, and sweep grids used in every numerical experiment; the remaining appendices provide complete proofs and additional quantitative and visual results. The accompanying source archive includes the shared JAX/Equinox implementations in \nolinkurl{src/architecture.py} and \nolinkurl{src/moment_matching.py}; \nolinkurl{src/cameraman-and-audio.ipynb} reproduces the image, narrow-network, and audio experiments; \nolinkurl{src/omega-sweep.ipynb} and \nolinkurl{src/mup-lr-sweep.ipynb} reproduce the two hyperparameter studies; \nolinkurl{src/jacobian_svd_spectrum (CVP26 comparison).ipynb} reproduces the spectral diagnostic; and \nolinkurl{src/image_fitting-repro.ipynb} records the CVP26 image-fitting replication used to validate the shared implementation. The image data are obtained through \texttt{scikit-image}, and the Bach waveform is included at \nolinkurl{src/assets/gt_bach.wav}. The two sweep notebooks are configured for accelerator-backed cloud execution and write to mounted \nolinkurl{/mnt} directories; for local execution, the user must select an available JAX backend and change those output directories to writable local paths.

\bibliography{stable-initialization}
\bibliographystyle{iclr2027_conference}

\appendix
\let\appendixsection\section
\renewcommand{\section}{\FloatBarrier\appendixsection}

\section{Comparison to conventional stable initialization of hidden layers}
\label{app:cvp26-comparison}

We compare our hidden-layer initialization to that of \citet{combette_new_2026}, which represents a state-of-the-art implementation of the conventional wisdom for stable initialization.
Pursuant to the four-step organization described in the Introduction,
there are two significant derivations.

\subsection{Preactivation distribution}

The first major effort is to approximate preactivations \(z_\ell\) as asymptotically Normal with a tunable variance \(\sigma_a^2\):
\(z_\ell \to \mathcal{N}(0, \sigma_a^2).\)

\begin{priorwork}[Thm.~3.1, \cite{combette_new_2026}]
	Consider the sinusoidal network defined above, where for some $c_w, c_b \in \mathbb{R}^+$, and for every layer $\ell \in \{2, \ldots, L\}$, the weight matrix $W_\ell$ is initialized with entries sampled from $\mathcal{U}(-c_w/\sqrt{N},\, c_w/\sqrt{N})$, $W_1$ is sampled from $\mathcal{U}(-w_0/n_0,\, w_0/n_0)$, and the bias $b_\ell$ is initialized with entries sampled from $\mathcal{N}(0, c_b^2)$.
	Let $(z_\ell)_{\ell \in \{1, \ldots, L\}}$ be the preactivation sequence $z_\ell = W_\ell x_{\ell-1} + b_\ell$ for an input $x \in \mathbb{R}^{n_0}$.
	Then, in the limits $N, L \to \infty$, the preactivation sequence $(z_\ell)_{\ell \in \mathbb{N}}$ converges in distribution to $\mathcal{N}(0, \sigma_a^2)$ with
	\begin{equation}
		\sigma_a^2 = c_b^2 + \frac{c_w^2}{6} + \frac{1}{2} \mathcal{W}_{0}\!\left(-\frac{c_w^2}{3} e^{-\frac{c_w^2}{3} - 2 c_b^2}\right), \label{eq:sigma_a}
	\end{equation}
	where $\mathcal{W}_{0}$ is the principal real branch of the Lambert $W$ function.
	The sequence $\bigl(\Var(z_\ell)\bigr)_{\ell \in \mathbb{N}}$ converges to the fixed point $\sigma_a^2$, which is exponentially attractive for all $c_w \neq \sqrt{3}$.
	For $c_w = \sqrt{3}$, the convergence is of rate $\mathcal{O}(1/\ell)$.
\end{priorwork}
The limit \(N \to \infty\) is necessary in order to apply the Central Limit Theorem to establish that the preactivations converge weakly to an isotropic Normal distribution.
(Because the sine activation and its derivatives are bounded and continuous, weak convergence is sufficient to reason about the network's moments.)
The limit \(L \to \infty\) is necessary in order to wash out the transient effect of the embedding layer on the preactivation distribution.
An elegant nonlinear recursion is derived for each preactivation layer's moments in terms of those of the previous layer;
this recursion has an attractive fixed point given by
\eqref{eq:sigma_a}, reflecting a delicate balance between weights and biases.

All of this follows from the need to enforce a Normality Ansatz.
While this Ansatz would make sense for nonperiodic activation functions,
the parallel claim in our result is rendered nearly trivial by the fact that the sine function's domain can, without loss of generality, be wrapped onto a torus.
\begin{ourresult}
	Let \(b_\ell\) be uniformly distributed on the \(N\)-torus.
	Let \(z_\ell = W_\ell x_{\ell-1} + b_\ell\).
	Then \(z_\ell \bmod 2\pi\) is uniformly distributed on the \(N\)-torus.
\end{ourresult}
Adding a uniform phase wipes out the distribution of \(W_\ell x_{\ell-1}\).
This eliminates the need for recursive analysis of the wrapped preactivation distribution because it is toroidally uniform by construction.
There is also no need to take \(L \to \infty\), because we have decoupled the layers and imposed an exactly uniform distribution on every layer's wrapped preactivation.
Furthermore, there is no need to take \(N \to \infty\), because we do not rely on the Central Limit Theorem.
Our wrapped preactivation distribution is exact.

Finally, we observe that \citet[Thm.~3.1]{combette_new_2026} is compatible with the key idea of our theory.
Keeping \(c_w\) finite, take the limit \(c_b^2 \to \infty\).
Using \(\mathcal W_0(0) = 0\),
\eqref{eq:sigma_a}
reduces in this limit to
\begin{align}
	\sigma_a^2 = c_b^2 = \infty.
\end{align}
When a Normal distribution with infinite variance is wrapped onto the torus, the result is a uniform distribution, recovering our finding that the preactivation distribution can be decoupled from the weight scale \(c_w\).

\subsection{Activation gradient distribution}
Using the preactivation distribution, one can now estimate the quadratic mean of the activation function's slope over the preactivation ensemble through the more general identity
\begin{align}
	\Var \sbr{\dpd{}{z_\ell}\sin(z_\ell)}
	 & = \expect\sbr{\cos^2(z_\ell)} = \frac{1 + e^{-2\sigma_a^2}}{2} \quad \text{elementwise}.
\end{align}
This principle may be used to analyze the layer Jacobian entrywise:
\begin{priorwork}[Thm.~3.2, \cite{combette_new_2026}] \label{thm-gradient-distribution}
	Let $\partial x_\ell / \partial x_{\ell-1}$ denote the Jacobian of the $\ell$-th layer, so that
	\begin{equation}
		\dpd{x_\ell}{x_{\ell-1}} = \mathrm{diag}(\cos(z_\ell)) \, W_\ell.
	\end{equation}
	Under the same assumptions as above, and maintaining the limit of large $N$, each entry of this Jacobian has zero mean and variance $\widetilde{\sigma}_\ell^2$, such that the sequence $(N \widetilde{\sigma}_\ell^2)_{\ell \in \mathbb{N}}$ converges to
	\begin{equation}
		\lim_{\ell, N \to \infty} (N \widetilde{\sigma}_\ell^2) = \sigma_g = \frac{c_w^2}{6} \bigl(1 + e^{-2 \sigma_a^2}\bigr). \label{eq:sigma_g}
	\end{equation}
\end{priorwork}
A parallel result can also be stated for our initialization:
\begin{ourresult}
	Let the wrapped preactivation \(z_\ell \bmod 2\pi\) be uniformly distributed on the \(N\)-torus.
	Assume that each entry of \(W_\ell\) has zero mean and variance \(c_w^2 / (3N)\).
	Then each entry of the Jacobian
	\begin{align}
		\dpd{x_\ell}{x_{\ell-1}} = \mathrm{diag}(\cos(z_\ell)) W_\ell
	\end{align}
	has zero mean and variance \(c_w^2/(6N)\).
\end{ourresult}
Note again the lack of asymptotics.
Activation-slope vectors such as \(\cos(z_\ell)\) are i.i.d.

\subsection{What gradient criterion?}
Finally, we turn to the choice of criterion to stabilize (Step 1 of the Introduction).
Whereas \citet{combette_new_2026} stabilize the \emph{entries} of the Jacobian, we stabilize its gramian matrix.
Translated into our convention, their result (Appendix A.4) may be summarized as follows.
\begin{priorwork}[Appendix A.4, \cite{combette_new_2026}]\label{priorwork:gradient-scaling}
	Let $f$ be a sinusoidal network initialized so that each entry of the layerwise Jacobian $\partial x_\ell / \partial x_{\ell-1}$ has variance $\sigma_g^2 / N$.
	Let $\Psi$ be a scalar loss function applied to the network output.
	Then the parameter-wise and input-wise gradient variances scale as
	\begin{align}
		\Var\sbr{\nabla_{W_{\ell,i,j}} \Psi(f(x))} & \sim \frac{(N \sigma_g^2)^{L-\ell-1}}{N}, & \ell & > 1, \\
		\Var\sbr{\nabla_{x_i} \Psi(f(x))}          & \sim (\sigma_g^2)^{L-2}.
	\end{align}
	In particular, gradients vanish or explode exponentially with depth unless $N \sigma_g^2 \approx 1$.
\end{priorwork}
This scaling analysis is a plausible---if not fully rigorous (it applies an unjustified independence heuristic to random matrix-vector products)---corollary of the elementwise Jacobian result of \citet[Thm.~3.2]{combette_new_2026}.
However, it relies on the double limit \(\ell, N \to \infty\) and hides an unknown constant.

Our initialization removes the asymptotic limits from the Jacobian second-moment calculation.
Translating that calculation into a loss-gradient prediction requires an additional assumption on the loss cotangent.
\begin{ourresult}
	Let $f$ be a sinusoidal network with the uniform-phase initialization, let \(2 \leq \ell \leq L\), and let \(g = \nabla_{x_L} \Psi(f(x)) \in \mathbb R^N\) denote the loss cotangent at the output of the hidden map.
	Then the parameter and input gradients decompose as
	\begin{align}
		\nabla_{W_\ell} \Psi(f(x)) & = \Cos(z_\ell) \del{\dpd{x_L}{x_\ell}}^\intercal  g  x_{\ell-1}^\intercal
		\quad \text{and}                                                                                       \\
		\nabla_{x_0} \Psi(f(x))    & = \del{\dpd{x_L}{x_0}}^\intercal g.
	\end{align}
	Using an independence heuristic,
	\begin{align}
		\expect \del{\nabla_{W_\ell} \Psi(f(x))}^\intercal \del{\nabla_{W_\ell} \Psi(f(x))}
		                                                                              & \approx \frac{1}{4} \left\|g\right\|^2 I
		\quad\text{and}
		\label{eq:weight-gradient-gramian}
		\\
		\expect \del{\nabla_{x_0} \Psi(f(x))}^\intercal \del{\nabla_{x_0} \Psi(f(x))} & \approx \left\|g\right\|^2.
	\end{align}
\end{ourresult}
\begin{proof}[Proof of \eqref{eq:weight-gradient-gramian}]
	Most of the independence we need is rigorous (Lemma~\ref{lem:independent-factorization}).
	The remaining approximation treats \(g\) as approximately independent of the hidden-layer initialization \citep{chickering_spectral_2025,yang_spectral_2024}.
	As in \citet{combette_new_2026}, this is an early-training heuristic: for a general loss, the cotangent depends on the initialized network.
	\begin{align*}
		\expect \del{\nabla_{W_\ell} \Psi(f(x))} ^\intercal \del{\nabla_{W_\ell} \Psi(f(x))}
		 & =
		\expect x_{\ell-1} g^\intercal
		\del{\dpd{x_L}{x_\ell}}
		\Cos(z_\ell)
		\Cos(z_\ell) \del{\dpd{x_L}{x_\ell}}^\intercal g  x_{\ell-1}^\intercal
		\\
		 & \approx
		\frac{1}{2} \left\|g\right\|^2
		\expect x_{\ell-1}
		x_{\ell-1}^\intercal
		\\
		 & \approx
		\frac{1}{2} \left\|g\right\|^2
		\expect \sin(z_{\ell-1}) \sin(z_{\ell-1})^\intercal \\
		 & \approx
		\frac{1}{4} \left\|g\right\|^2 I
	\end{align*}
\end{proof}

\subsection{Spectral diagnostics}
We complement the theory of Section~\ref{sec:theory} with mechanistic diagnostics that illuminate how the uniform-phase initialization differs from its predecessors: the singular value spectrum of the end-to-end Jacobian and the visual appearance of random initialization fields.
Although Theorem~\ref{thm:general} and CVP26's asymptotic analysis make different predictions about the distribution of singular values, neither fully determines the spectral shape.
Therefore we inspect it empirically.

We extend \citet[Figure 8]{combette_new_2026}, which plots the singular values of \(J = \dpd{f}{x}\) from greatest to least, by adding the uniform-phase initialization.
To ensure a fair comparison, the input and output layers of our moment-matched uniform-phase initialization are calibrated to match the output statistics of a CVP26 (\(\sigma_a = 0\), \(L=4\)) network, so that the only variable is the hidden-layer initialization.

The uniform-phase initialization exhibits spectral broadening with depth: the largest singular values grow (``rich get richer'') while the smallest decay toward zero (``poor get poorer'').
This is a way of satisfying Theorem~\ref{thm:general}, which fixes the \emph{average} squared singular value at unity while leaving the distribution's shape unconstrained.
If some singular values grow, others must shrink to balance the mean.

This contrasts sharply with CVP26's \(\sigma_a = 0\) initialization, which instead preserves the \emph{largest} singular value (the operator norm) across depth.
There, the ``poor get poorer'' phenomenon occurs for a different reason: the operator norm is held fixed while the average squared singular value decays, so the network's mean amplification shrinks even as its peak amplification holds steady.
The two methods thus embody different design goals: CVP26 \(\sigma_a{=}0\) stabilizes the worst case, while the uniform-phase initialization stabilizes the average case.

The mechanism is also very different.
CVP26 concentrates preactivations on the steep portion of the sine, amplifying slopes uniformly, whereas the uniform-phase initialization averages over the full sine (including its stationary and downward-sloping regions).
The effect is most visible in comparison to SM20, which uses the same weight distribution as the uniform-phase initialization but smaller biases, and consequently suffers gradient blowup.

\begin{figure}[ht]
	\centering
	\includegraphics[width=\textwidth]{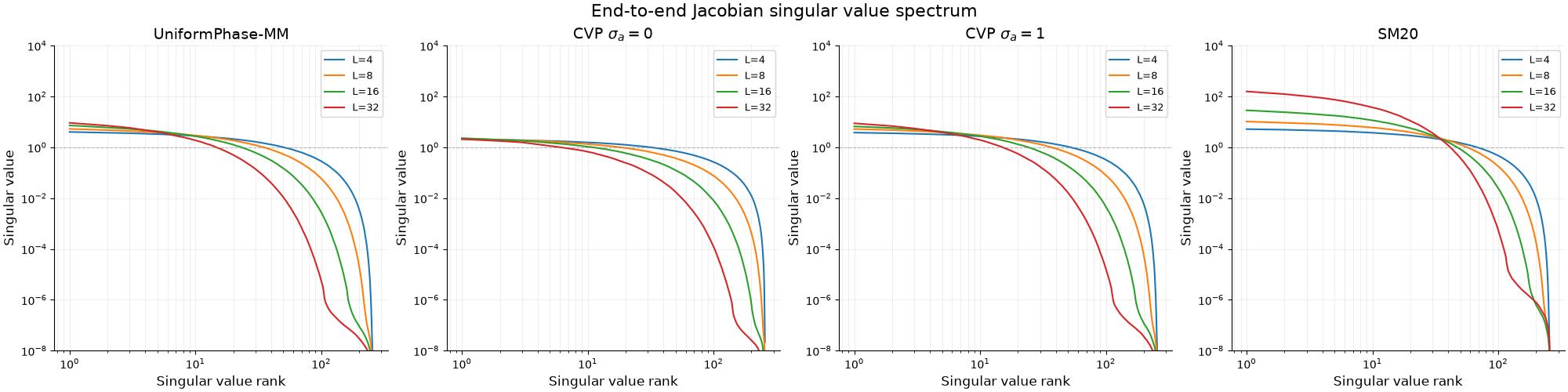}
	\caption{Full singular value spectrum of the end-to-end Jacobian \(\mathbf{J} = \partial\mathbf{h}_L/\partial\mathbf{h}_1\) as a function of depth. Four initialization schemes are compared: uniform-phase MM (moment-matched to CVP26 \(\sigma_a{=}0\), \(L{=}4\)), CVP26 \(\sigma_a{=}0\), CVP26 \(\sigma_a{=}1\), and the original SM20 initialization. Each spectrum is averaged over five independently initialized networks, with 10 sample points on \([-\pi, \pi]\).}
	\label{fig:svd-cvp26}
\end{figure}

\section{Supplement to \S\ref{sec:mup}}\label{app:mup-supplement}
\begin{remark}[\(\mu\)P last-layer scaling]\label{rem:mup-last-layer}
	Whereas moment matching scales the last-layer weights by \(N^{-1/2}\) (\S\ref{sec:mup}), \(\mu\)P insists on scaling them by \(N^{-1}\), which causes the network output to decay as \(N^{-1/2}\).
	See \citet[Appendix J.2.1]{yang_tensor_2022} for the derivation, and Appendix D.2, ibid., for a justification of output weight collapse.
	We note in passing that last-layer weights are intuitively easier to learn because training the last layer while holding other parameters fixed is a convex problem;
	nevertheless, in empirical training, the representation learning task can be sensitive to output scaling.
\end{remark}

\section{Tasks and experimental methods}\label{app:experimental-details}

\paragraph{Common protocol and metrics.}
The fitting experiments represent a sampled field with a coordinate MLP and minimize mean squared error (MSE) over all samples and output dimensions.
Every optimizer step is full-batch.
Model parameters and training arrays use 32-bit floating-point values, while moment estimates are accumulated in 64-bit NumPy and then converted to model precision.
Adam is used without weight decay, gradient clipping, or a learning-rate schedule.
For image and audio fitting we report
\begin{equation}
	\operatorname{SNR}=10\log_{10}\frac{\operatorname{Var}(y)}{\operatorname{MSE}(\hat y,y)},
\end{equation}
where the variance, like the MSE, is computed over the complete evaluation array.
The high-resolution image grid and the full audio signal are evaluation sets rather than independently sampled observations.
We present the image artifacts with matplotlib-default lossy compression in order to reduce the size of this file, and certify that it does not materially affect their interpretation.

\paragraph{Initialization implementations.}
For uniform-phase MM, let a regularly sampled training field contain $m$ values $y_i\in\mathbb R^{N_{\mathrm{out}}}$ and let $J_i$ denote its spatial Jacobian.
We use the sample mean and covariance and form
\begin{equation}
	\widehat\Sigma_\epsilon=\widehat\Sigma+\epsilon I,
	\qquad
	\widehat\Omega
	=\frac{1}{m}\sum_{i=1}^{m}J_i^\intercal
	\widehat\Sigma_\epsilon^{-1}J_i+\epsilon I,
	\qquad \epsilon=10^{-6}.
\end{equation}
Each $J_i$ is estimated channel by channel with \texttt{scipy.ndimage.sobel}; the result is divided by the coordinate spacing and by $2\cdot4^{d-1}$, which accounts for the centered difference and smoothing in the other $d-1$ spatial directions.
The input weights have independent rows distributed as $\mathcal N(0,\widehat\Omega/N_{\mathrm{out}})$, the input phase and every hidden bias are independent $\mathcal U(-\pi,\pi)$ variables, and hidden weights are independent $\mathcal N(0,2/N)$ variables.
The output bias is $\widehat\mu$ and the columns of the output matrix are independent $\mathcal N(0,2\widehat\Sigma_\epsilon/N)$ variables.

The SM20 and CVP26 implementations follow the corresponding PyTorch releases \citep{sitzmann_implicit_2020,combette_new_2026} but are written in JAX/Equinox.
Their first-layer weights are sampled uniformly on $[-1/N_{\mathrm{in}},1/N_{\mathrm{in}}]$ and then multiplied by the reported first-layer $\omega_0$ inside the sine.
For subsequent sine layers,
$W_{ij}\sim\mathcal U[-c/(\omega_h\sqrt N),c/(\omega_h\sqrt N)]$ and the activation is $\sin(\omega_h(Wx+b))$.
SM20 uses $c=\sqrt6$ and PyTorch-style uniform biases; CVP26 $\sigma_a=0$ uses $c=\sqrt3$ and zero hidden biases; and CVP26 $\sigma_a=1$ uses $c=\sqrt{6/(1+e^{-2})}$ and centered Normal hidden biases of variance $c^2e^{-2}/3$ when $\omega_h=1$.
The final linear map has no bias and has weights uniform on $[-\sqrt{3/N}/\omega_h,\sqrt{3/N}/\omega_h]$.
The Tanh (PE-Xavier) baseline encodes $x$ as $[\sin(2\pi Bx),\cos(2\pi Bx)]$, using $N/2$ rows of $B$ sampled independently from $\mathcal N(0,\sigma^2I)$, and applies a Xavier-initialized Tanh MLP.
The ReLU baseline uses Kaiming-uniform weights, while the SiLU and GeLU baselines use Xavier-uniform weights; these MLPs use zero biases.
All baselines match the compared SIRENs in width and depth.

\paragraph{Cameraman image fitting and narrow networks.}
We take \texttt{skimage.data.camera}, downsample it from its native $512\times512$ resolution to $128\times128$ with anti-aliasing, and associate pixels with an endpoint-inclusive Cartesian grid on $[-1,1]^2$.
Training uses all $16{,}384$ low-resolution coordinate--value pairs; evaluation uses all $262{,}144$ native-resolution pairs.
Pixel intensities are represented in $[0,1]$.
The standard experiment uses the architecture and hyperparameters of \citet{combette_new_2026}'s supplementary material: $L=11$, $N=256$, 3000 Adam steps, and learning rate $10^{-4}$; the narrow experiment changes only the width to $N=16$.
The first-layer frequency is $64\pi/\sqrt6$ for SM20 and CVP26, and hidden frequencies are 1.
The positional-encoding scale is $\sigma=32$.

\paragraph{Audio fitting.}
The bundled \texttt{gt\_bach.wav} file is loaded at its native 44.1\,kHz sampling rate, converted to mono, and normalized by its maximum absolute amplitude.
It contains 308,207 samples (approximately 7\,s), which are placed on an endpoint-inclusive grid on $[-100,100]$.
The full signal is used both to optimize and to compute MSE and SNR; there is no held-out or subsampled split.
All models use $L=5$, $N=256$, 5000 full-batch Adam steps, and learning rate $5\times10^{-5}$.
Both baseline SIRENs use first-layer $\omega_0=30$; CVP26 uses hidden $\omega_h=1$, whereas the SM20 replication uses $\omega_h=30$ as in its reference implementation.
The Tanh positional-encoding scale is $\sigma=10$.

\paragraph{Input-frequency sweep.}
For each image and for each of SM20, CVP26 $\sigma_a=0$, and CVP26 $\sigma_a=1$, we evaluate
\begin{equation}
	\omega_0=\frac{64\pi}{\sqrt6}\,10^q,
	\qquad q\in\{-2,-1.8,\ldots,1.8,2\}.
\end{equation}
A baseline uses the same random weights at all 21 frequencies, so frequency is the only within-method variable; uniform-phase MM is initialized and trained once per image.
All runs use $L=11$, $N=256$, and 3000 full-batch Adam steps at learning rate $10^{-4}$, with training and evaluation grids whose largest dimensions are 128 and 512, respectively, and whose aspect ratio is preserved.
The grayscale tasks use \texttt{brick}, \texttt{camera}, \texttt{gravel}, and \texttt{clock}; the color tasks use \texttt{hubble\_deep\_field}, \texttt{skin}, \texttt{astronaut}, and \texttt{coffee} from \texttt{scikit-image}.
RGB values are scaled to $[0,1]$.
For the CIELAB variant, images are resized in RGB and then transformed with \texttt{skimage.color.rgb2lab}; training and MSE evaluation take place in CIELAB, and conversion back to RGB is used only for display.

\paragraph{Width and learning-rate sweep.}
The $\mu$P experiment uses the grayscale Cameraman training field and reports no test loss.
First, at base width $N_0=64$, each prior SIREN initialization is trained at the 21 frequencies above using a learning rate of $10^{-4}$; its frequency is selected by the mean of the 3000 pre-update training losses.
Uniform-phase MM requires no such selection.
Second, for $N\in\{8,16,32,64,128,256,512\}$, each initialization is trained at each $\eta\in\{2^{-18},2^{-17},\ldots,2^2\}$ for 3000 steps.
The plotted criterion is again the time average of the per-step pre-update training MSE, rather than final MSE.
All learning rates for a given initialization and width begin from identical parameters.
In this experiment Adam uses $\epsilon=0$ so that its additive numerical constant does not introduce a width-dependent scale.
Under \textbf{$\mu$P (LR)}, hidden matrices and the output projection use an effective learning rate of $\eta N_0/N$, while input projections and all biases use $\eta$.
Under \textbf{$\mu$P (LR+LL)}, the same optimizer scaling is used and the initialized output weights are additionally multiplied by $\sqrt{N_0/N}$.
The \textbf{standard} condition uses $\eta$ for every parameter and leaves initialization unchanged.

\paragraph{Jacobian-spectrum diagnostic.}
Following the protocol of \citet[Appendix~B.1]{combette_new_2026}, we study the hidden-to-hidden Jacobian $J=\partial h_L/\partial h_1$ for width $N=256$ and depths $L\in\{4,8,16,32\}$.
The uniform-phase input and output moments are estimated from evaluations of a CVP26 $\sigma_a=0$, $L=4$ network with first-layer $\omega_0=30$ and hidden $\omega_h=1$ at 1000 evenly spaced points on $[-\pi,\pi]$.
For each initialization and depth, we compute sorted singular values at 10 uniformly sampled input points for each of five independently initialized networks and average each ordered singular value over the resulting 50 spectra.

\section{Additional experimental results}\label{app:image-fitting-results}
\begin{table}[ht]
	\centering
	\caption{Quantitative results for the Cameraman image-fitting experiment (\S\ref{sec:image-fitting}). Train MSE is computed on the $128 \times 128$ grid; eval MSE and SNR are computed on the $512 \times 512$ grid. Best values are in bold.}
	\label{tab:cameraman}
	\begin{tabular}{lrrr}
		\toprule
		Network                           & Train MSE           & Eval MSE            & SNR (dB)          \\
		\midrule
		\textbf{SIREN (uniform-phase MM)} & $\mathbf{0.000032}$ & $\mathbf{0.002706}$ & $\mathbf{14.89}$  \\
		SIREN (CVP26 $\sigma_a\!=\!0$)    & $0.000127$          & $0.002967$          & $14.49$           \\
		SIREN (SM20)                      & $0.000000$          & $0.017428$          & $\phantom{0}6.80$ \\
		ReLU (Kaiming)                    & $0.001768$          & $0.005201$          & $12.05$           \\
		SiLU (Xavier)                     & $0.007591$          & $0.011650$          & $\phantom{0}8.55$ \\
		GeLU (Xavier)                     & $0.010906$          & $0.015013$          & $\phantom{0}7.45$ \\
		Tanh (PE-Xavier)                  & $0.000274$          & $0.124271$          & $-1.73$           \\
		\bottomrule
	\end{tabular}
\end{table}

\begin{table}[ht]
	\centering
	\caption{Quantitative results for the narrow (\(N=16\)) Cameraman experiment (\S\ref{sec:image-fitting}).}
	\label{tab:cameraman-narrow}
	\begin{tabular}{lrrr}
		\toprule
		Network                           & Train MSE           & Eval MSE            & SNR (dB)          \\
		\midrule
		\textbf{SIREN (uniform-phase MM)} & $\mathbf{0.004117}$ & $\mathbf{0.007824}$ & $\mathbf{10.28}$  \\
		ReLU (Kaiming)                    & $0.006713$          & $0.010689$          & $\phantom{0}8.92$ \\
		GeLU (Xavier)                     & $0.017129$          & $0.021263$          & $\phantom{0}5.94$ \\
		SiLU (Xavier)                     & $0.018928$          & $0.023064$          & $\phantom{0}5.58$ \\
		SIREN (CVP26 $\sigma_a\!=\!0$)    & $0.033747$          & $0.037765$          & $\phantom{0}3.44$ \\
		SIREN (SM20)                      & $0.062279$          & $0.068359$          & $\phantom{0}0.86$ \\
		Tanh (PE-Xavier)                  & $0.073832$          & $0.088348$          & $-0.25$           \\
		\bottomrule
	\end{tabular}
\end{table}

\begin{figure}[ht]
	\centering
	\includegraphics[width=\textwidth]{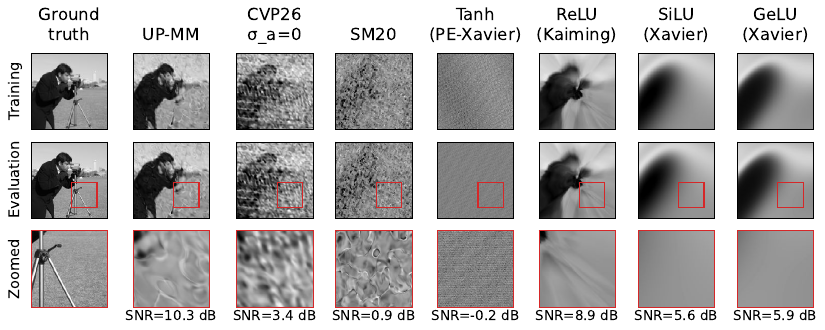}
	\caption{Image fitting with narrow networks (\(N=16\)). Layout as in \figref{fig:cameraman}.}
	\label{fig:cameraman-narrow}
\end{figure}

\begin{figure}[ht]
	\centering
	\includegraphics[width=\textwidth]{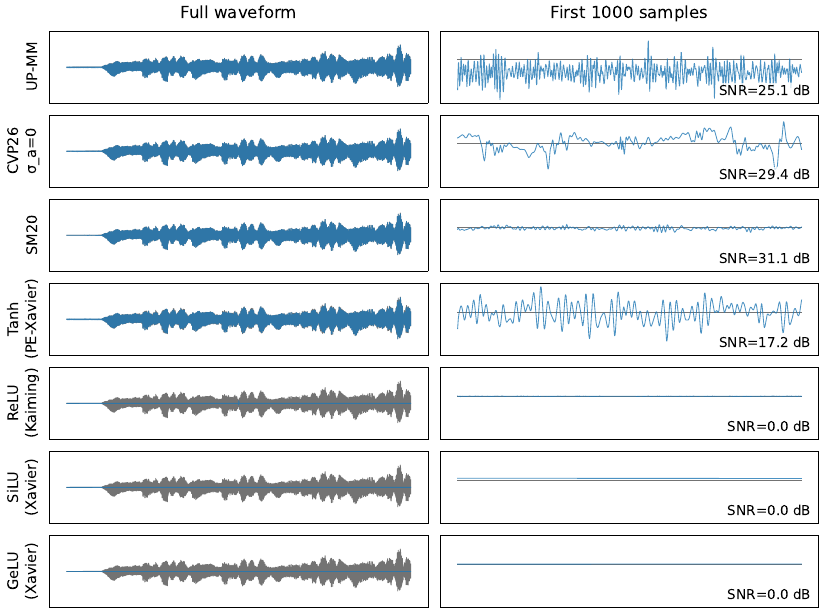}
	\caption{Audio fitting on the Bach recording. One row per initialization, with the ground truth in gray and the prediction in blue. Left: the full waveform, drawn as a per-pixel minimum/maximum envelope of all $308{,}207$ samples. Right: the first $1000$ samples, on its own vertical scale because the recording opens in near-silence. The evaluation SNR is given in each right-hand panel.}
	\label{fig:audio}
\end{figure}

\begin{figure}[ht]
	\centering
	\includegraphics[width=\textwidth]{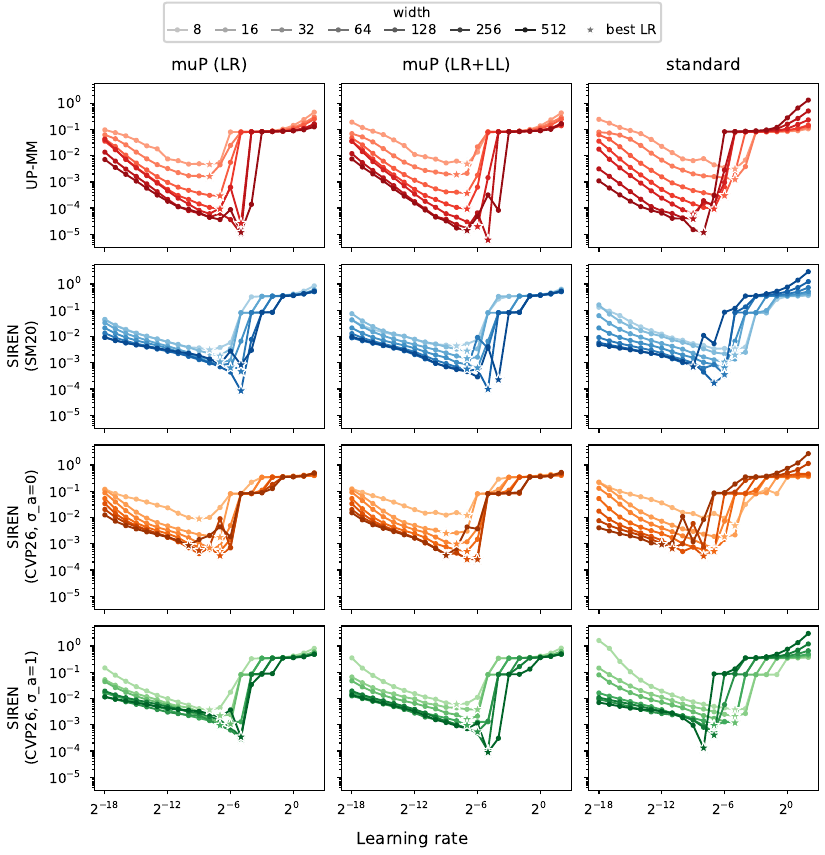}
	\caption{\(\mu\)P learning-rate transfer, grayscale Cameraman, geometrically time-averaged train MSE (\S\ref{sec:mup}). Rows: initialization (uniform-phase MM and the three SIREN baselines at their best \(\omega_0\)); columns: \(\mu\)P LR, \(\mu\)P LR+LL, and standard scaling. One curve per width, shaded from light (\(N=8\)) to dark (\(N=512\)), with a star at each width's best learning rate.}
	\label{fig:mup-lr-train}
\end{figure}

\section{Moment-matched parameters and best-\texorpdfstring{\(\omega_0\)}{omega0} test errors}\label{app:omega-sweep}

\paragraph{Grayscale}
The estimated moment-matched parameters are:
\begin{align}
	\hat\mu & = 0.437 & \hat\Sigma & = 0.00672 & \hat\Omega & = \begin{pmatrix} 561 & -4.2 \\ -4.2 & 1870 \end{pmatrix} \tag{Brick}     \\
	\hat\mu & = 0.506 & \hat\Sigma & = 0.0793  & \hat\Omega & = \begin{pmatrix} 87.0 & 10.9 \\ 10.9 & 115 \end{pmatrix} \tag{Camera}    \\
	\hat\mu & = 0.496 & \hat\Sigma & = 0.0125  & \hat\Omega & = \begin{pmatrix} 1204 & -48.6 \\ -48.6 & 1120 \end{pmatrix} \tag{Gravel} \\
	\hat\mu & = 0.574 & \hat\Sigma & = 0.00664 & \hat\Omega & = \begin{pmatrix} 56.5 & 0.53 \\ 0.53 & 30.7 \end{pmatrix} \tag{Clock}
\end{align}

\begin{table}[ht]
	\centering
	\caption{Grayscale sweep: best test MSE per method. The optimal \(\omega_0\) from the grid search is shown in parentheses. Bold entries are the best per row. Visual comparisons: \figref{fig:best-grayscale}.}
	\label{tab:omega-sweep-gray}
	\begin{tabular}{lcccc}
		\toprule
		Image                           & uniform-phase MM  & SM20                                         & CVP26 \(\sigma_a{=}0\)              & CVP26 \(\sigma_a{=}1\)                       \\
		\midrule
		Brick   & 0.002852          & \textbf{0.002317} {\tiny(\(\omega{=}20.6\))} & 0.002408 {\tiny(\(\omega{=}82.1\))} & 0.002381 {\tiny(\(\omega{=}32.7\))}          \\
		Camera & \textbf{0.002695} & 0.002763 {\tiny(\(\omega{=}20.6\))}          & 0.002912 {\tiny(\(\omega{=}130\))}  & 0.002855 {\tiny(\(\omega{=}51.8\))}          \\
		Gravel & 0.009063          & 0.007982 {\tiny(\(\omega{=}32.7\))}          & 0.007994 {\tiny(\(\omega{=}130\))}  & \textbf{0.007937} {\tiny(\(\omega{=}51.8\))} \\
		Clock   & \textbf{0.000036} & 0.000055 {\tiny(\(\omega{=}20.6\))}          & 0.000050 {\tiny(\(\omega{=}82.1\))} & 0.000042 {\tiny(\(\omega{=}20.6\))}          \\
		\bottomrule
	\end{tabular}
\end{table}

\paragraph{Color (RGB)}
The estimated moment-matched parameters are:
\begin{align}
	\hat\mu & = \begin{pmatrix} 0.073 \\ 0.078 \\ 0.075 \end{pmatrix} & \hat\Sigma & = \begin{pmatrix} 0.0068 & 0.0054 & 0.0056 \\ 0.0054 & 0.0047 & 0.0050 \\ 0.0056 & 0.0050 & 0.0054 \end{pmatrix} & \hat\Omega & = \begin{pmatrix} 2323 & -14.8 \\ -14.8 & 3010 \end{pmatrix} \tag{Hubble}  \\
	\hat\mu & = \begin{pmatrix} 0.809 \\ 0.660 \\ 0.743 \end{pmatrix} & \hat\Sigma & = \begin{pmatrix} 0.0031 & 0.0052 & 0.0030 \\ 0.0052 & 0.0165 & 0.0095 \\ 0.0030 & 0.0095 & 0.0059 \end{pmatrix} & \hat\Omega & = \begin{pmatrix} 699 & 23.9 \\ 23.9 & 888 \end{pmatrix} \tag{Skin}        \\
	\hat\mu & = \begin{pmatrix} 0.555 \\ 0.415 \\ 0.378 \end{pmatrix} & \hat\Sigma & = \begin{pmatrix} 0.097 & 0.075 & 0.066 \\ 0.075 & 0.084 & 0.082 \\ 0.066 & 0.082 & 0.086 \end{pmatrix}          & \hat\Omega & = \begin{pmatrix} 551 & -56.1 \\ -56.1 & 804 \end{pmatrix} \tag{Astronaut} \\
	\hat\mu & = \begin{pmatrix} 0.622 \\ 0.336 \\ 0.202 \end{pmatrix} & \hat\Sigma & = \begin{pmatrix} 0.057 & 0.046 & 0.032 \\ 0.046 & 0.052 & 0.042 \\ 0.032 & 0.042 & 0.038 \end{pmatrix}          & \hat\Omega & = \begin{pmatrix} 400 & 10.1 \\ 10.1 & 630 \end{pmatrix} \tag{Coffee}
\end{align}

\begin{table}[ht]
	\centering
	\caption{Color sweep: best test MSE per method. Conventions as in Table~\ref{tab:omega-sweep-gray}. Visual comparisons: \figref{fig:best-rgb}.}
	\label{tab:omega-sweep-color}
	\begin{tabular}{lcccc}
		\toprule
		Image                                 & uniform-phase MM  & SM20                                & CVP26 \(\sigma_a{=}0\)                      & CVP26 \(\sigma_a{=}1\)              \\
		\midrule
		Hubble       & 0.003773          & 0.003758 {\tiny(\(\omega{=}32.7\))} & \textbf{0.003715} {\tiny(\(\omega{=}130\))} & 0.003808 {\tiny(\(\omega{=}51.8\))} \\
		Skin           & \textbf{0.005784} & 0.006044 {\tiny(\(\omega{=}20.6\))} & 0.006083 {\tiny(\(\omega{=}130\))}          & 0.005870 {\tiny(\(\omega{=}51.8\))} \\
		Astronaut & \textbf{0.004027} & 0.004076 {\tiny(\(\omega{=}32.7\))} & 0.004276 {\tiny(\(\omega{=}130\))}          & 0.004191 {\tiny(\(\omega{=}51.8\))} \\
		Coffee       & \textbf{0.002728} & 0.002785 {\tiny(\(\omega{=}20.6\))} & 0.002847 {\tiny(\(\omega{=}82.1\))}         & 0.002874 {\tiny(\(\omega{=}51.8\))} \\
		\bottomrule
	\end{tabular}
\end{table}

\paragraph{Color (CIELAB)}
The estimated moment-matched parameters in CIELAB space are:
\begin{align}
	\hat\mu & = \begin{pmatrix} 6.47 \\ -0.35 \\ 0.16 \end{pmatrix}   & \hat\Sigma & = \begin{pmatrix} 64.4 & 12.3 & -0.21 \\ 12.3 & 7.55 & 2.56 \\ -0.21 & 2.56 & 5.50 \end{pmatrix}             & \hat\Omega & = \begin{pmatrix} 2328 & -9.7 \\ -9.7 & 3023 \end{pmatrix} \tag{Hubble}    \\
	\hat\mu & = \begin{pmatrix} 72.99 \\ 17.52 \\ -5.39 \end{pmatrix} & \hat\Sigma & = \begin{pmatrix} 90.97 & -99.67 & 36.77 \\ -99.67 & 130.15 & -39.38 \\ 36.77 & -39.38 & 24.02 \end{pmatrix} & \hat\Omega & = \begin{pmatrix} 728 & 27.0 \\ 27.0 & 936 \end{pmatrix} \tag{Skin}        \\
	\hat\mu & = \begin{pmatrix} 47.73 \\ 13.55 \\ 11.95 \end{pmatrix} & \hat\Sigma & = \begin{pmatrix} 843 & -0.11 & 53.5 \\ -0.11 & 302 & 235 \\ 53.5 & 235 & 317 \end{pmatrix}                  & \hat\Omega & = \begin{pmatrix} 571 & -54.0 \\ -54.0 & 822 \end{pmatrix} \tag{Astronaut} \\
	\hat\mu & = \begin{pmatrix} 44.39 \\ 26.60 \\ 32.84 \end{pmatrix} & \hat\Sigma & = \begin{pmatrix} 495 & -36.95 & 126.4 \\ -36.95 & 189 & 145.3 \\ 126.4 & 145.3 & 203.3 \end{pmatrix}        & \hat\Omega & = \begin{pmatrix} 418 & 8.58 \\ 8.58 & 646 \end{pmatrix} \tag{Coffee}
\end{align}
Compared with the RGB parameters above, the CIELAB output covariance \(\hat\Sigma\) has much larger entries (reflecting the wider dynamic range of CIELAB coordinates) and mixed-sign off-diagonal entries (reflecting the decorrelating effect of the opponency axes).
The structure tensor \(\hat\Omega\) is nearly unchanged, confirming that the spatial-frequency content is preserved under the color-space transformation.

\begin{table}[ht]
	\centering
	\caption{CIELAB sweep: best test MSE per method. MSE is computed in CIELAB coordinates. Conventions as in Table~\ref{tab:omega-sweep-gray}. Visual comparisons: \figref{fig:best-cielab}.}
	\label{tab:omega-sweep-lab}
	\begin{tabular}{lcccc}
		\toprule
		Image                                     & uniform-phase MM & SM20                              & CVP26 \(\sigma_a{=}0\)            & CVP26 \(\sigma_a{=}1\)            \\
		\midrule
		Hubble       & \textbf{18.148}  & 18.505 {\tiny(\(\omega{=}13.0\))} & 18.675 {\tiny(\(\omega{=}82.1\))} & 18.692 {\tiny(\(\omega{=}51.8\))} \\
		Skin           & \textbf{45.172}  & 134.00 {\tiny(\(\omega{=}13.0\))} & 145.37 {\tiny(\(\omega{=}8208\))} & 141.75 {\tiny(\(\omega{=}8.21\))} \\
		Astronaut & \textbf{20.011}  & 83.61 {\tiny(\(\omega{=}32.7\))}  & 72.30 {\tiny(\(\omega{=}82.1\))}  & 84.91 {\tiny(\(\omega{=}20.6\))}  \\
		Coffee       & \textbf{15.658}  & 81.19 {\tiny(\(\omega{=}20.6\))}  & 82.15 {\tiny(\(\omega{=}32.7\))}  & 88.28 {\tiny(\(\omega{=}32.7\))}  \\
		\bottomrule
	\end{tabular}
\end{table}

\clearpage

\section{Best-\texorpdfstring{\(\omega_0\)}{omega0} visual comparisons}\label{app:best-omega-figures}
Each figure in this section covers one experiment: rows are the four images, and columns are the ground truth, the uniform-phase MM fit, and the best-tuned (by \(\omega_0\)) fits of SM20, CVP26 \(\sigma_a{=}0\), and CVP26 \(\sigma_a{=}1\) from the sweep of Appendix~\ref{app:omega-sweep}. Each panel is annotated with its test MSE and, for the baselines, the \(\omega_0\) that achieved it.

\begin{figure}[ht]
	\centering
	\includegraphics[width=\textwidth]{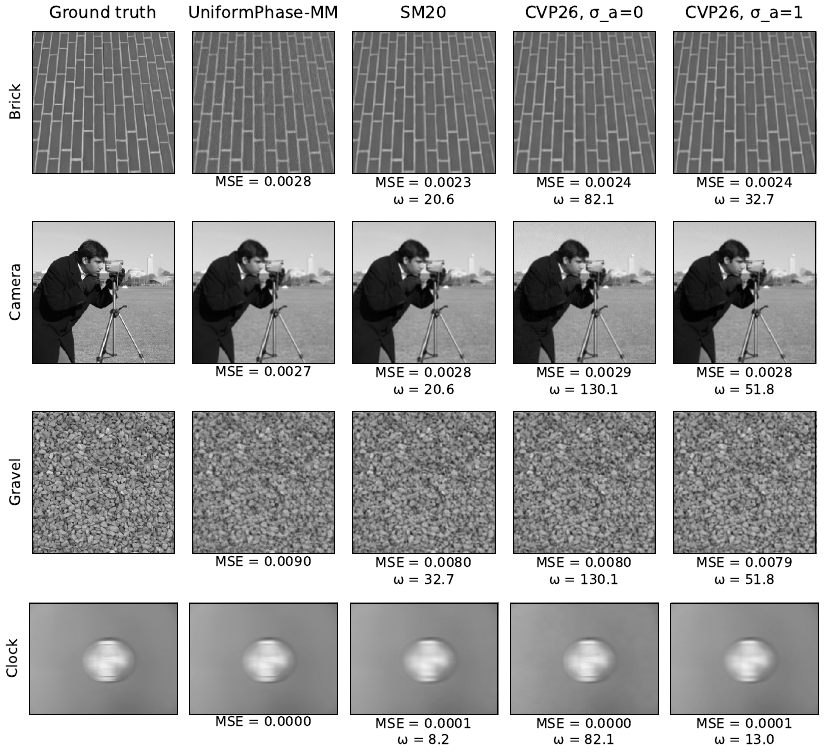}
	\caption{Grayscale best-\(\omega_0\) visual comparisons: Brick, Camera, Gravel, and Clock.}
	\label{fig:best-grayscale}
\end{figure}

\begin{figure}[ht]
	\centering
	\includegraphics[width=\textwidth]{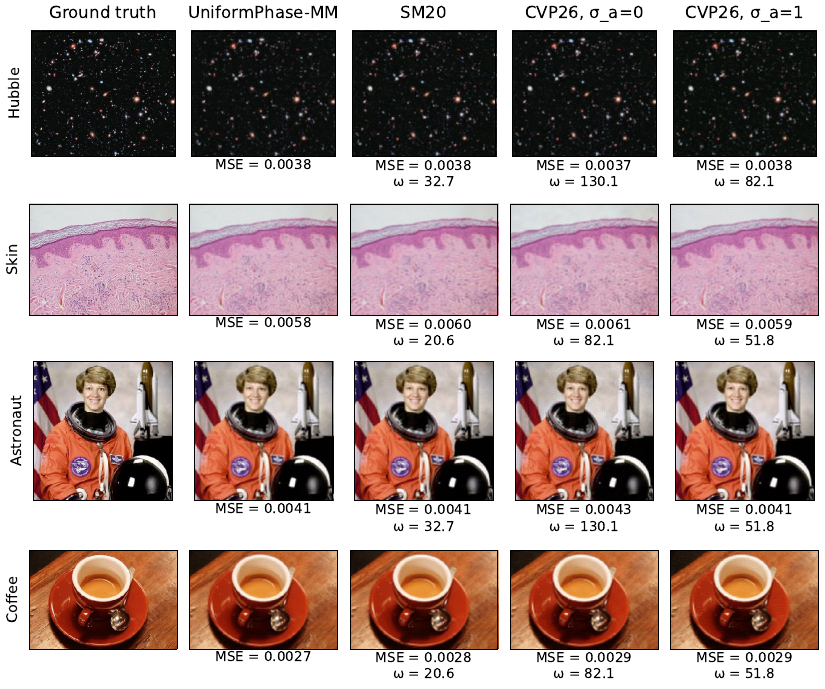}
	\caption{RGB best-\(\omega_0\) visual comparisons: Hubble, Skin, Astronaut, and Coffee. Same layout as \figref{fig:best-grayscale}.}
	\label{fig:best-rgb}
\end{figure}

\begin{figure}[ht]
	\centering
	\includegraphics[width=\textwidth]{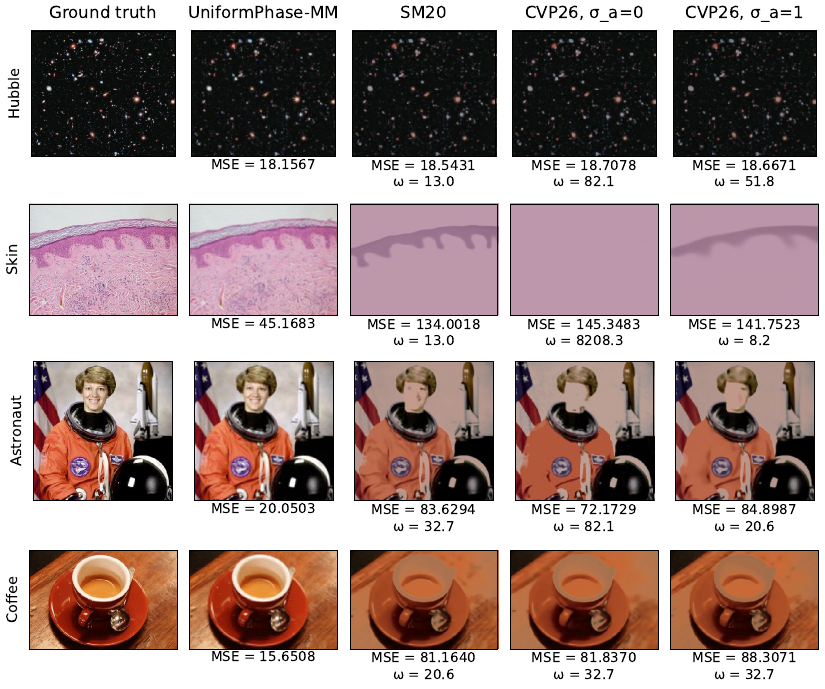}
	\caption{CIELAB best-\(\omega_0\) visual comparisons: Hubble, Skin, Astronaut, and Coffee. Predictions are converted from CIELAB to RGB for display; the reported MSE values are computed in CIELAB coordinates. Same layout as \figref{fig:best-grayscale}.}
	\label{fig:best-cielab}
\end{figure}

\end{document}

%% file: math_commands.tex
\usepackage{amsmath,amsfonts,bm}

\def\figref#1{figure~\ref{#1}}

\def\secref#1{section~\ref{#1}}

\def\eqref#1{equation~\ref{#1}}

\def\1{\bm{1}}

\DeclareMathAlphabet{\mathsfit}{\encodingdefault}{\sfdefault}{m}{sl}
\SetMathAlphabet{\mathsfit}{bold}{\encodingdefault}{\sfdefault}{bx}{n}

\newcommand{\Var}{\mathrm{Var}}

\newcommand{\Cov}{\mathrm{Cov}}
